\documentclass[lettersize,journal]{IEEEtran}
\usepackage{amsmath,amsfonts}
\usepackage{amsthm,amssymb,graphicx,multirow,amsmath,color,amsfonts}
\usepackage[ruled,vlined]{algorithm2e}
\usepackage{array}
\usepackage[caption=false,font=normalsize,labelfont=sf,textfont=sf]{subfig}
\usepackage{textcomp}
\usepackage{stfloats}
\usepackage{url}
\usepackage{verbatim}
\usepackage{graphicx}
\usepackage[shortlabels]{enumitem}
\usepackage{float}
\usepackage{bbm}
\usepackage{cite}
\newtheorem{assumption}{Assumption}
\newcommand*{\QEDA}{\hfill \ensuremath{\square}}
\newtheorem{lemma}{Lemma}

\newtheorem{definition}{Definition}

\usepackage{thmtools}
\usepackage{thm-restate}
\usepackage{enumitem}
\newlist{steps}{enumerate}{1}
\setlist[steps, 1]{label = Step \arabic*:}
\usepackage{hyperref}

\usepackage{cleveref}

\begin{document}

\title{Graph Exploration for Effective Multi-agent Q-Learning}

\author{Ainur Zhaikhan and Ali H. Sayed

\thanks{A. Zhaikhan and A. H. Sayed are with the Adaptive
Systems Laboratory, \'Ecole Polytechnique F\'ed\'erale de Lausanne (EPFL),
CH-1015, Switzerland.}}



\maketitle

\begin{abstract}
This paper proposes an exploration technique for multi-agent reinforcement learning (MARL) with graph-based communication among agents. We assume the individual rewards received by the agents are independent of the actions by the other agents, while their policies are coupled. In the proposed framework, neighbouring agents  collaborate to estimate the uncertainty about the state-action space in order to execute more efficient explorative behaviour. Different from existing works, the proposed algorithm does not require counting mechanisms and can be applied to continuous-state environments without requiring complex conversion techniques. Moreover, the proposed scheme allows agents to communicate in a fully decentralized manner with minimal information exchange. And for continuous-state scenarios, each agent needs to exchange only a single parameter vector. The performance of the algorithm  is verified with theoretical results for discrete-state scenarios and with experiments for continuous ones. 
\end{abstract}


\section{Introduction}
\label{sec:intro}

In reinforcement learning (RL), agents learn gradually through direct interaction with the environment. The successful operation of any RL algorithm is dependent on: i) how efficiently an agent is able to learn from collected data and ii) how efficiently it explores the design space and collects data. The former process is known as {\em exploitation}, while the latter is known as {\em exploration}  \cite{Exp_survey, Exp_survey2, sayed_2022} and it enables agents to explore the environment in a guided manner in order to collect informative data about the state, reward, and transition samples. Devising reliable exploration strategies is an important goal because they allow agents to  avoid being trapped into visiting a limited number of states and ignoring more rewarding possibilities. Encouraging agents to explore unseen states is essential for allowing successful learning in many cases of interest. \par 
Exploration in single agent RL has been intensively studied in the literature while most existing MARL algorithms  omit exploration by assuming it is easily obtainable through naive noise-based methods \cite{Q_mix, Rashid2,Alg1,Alg2}.  As explained in \cite{ISL}, such standard ``non-deep" exploration techniques  often fail in sparse reward problems. In order to tackle this issue, a couple of MARL algorithms have considered non-trivial exploration mechanisms \cite{justin, CMAE, LSVI_based}. However, some of these works cannot be directly applied to continuous-state problems \cite{justin}. This limitation comes from the fact  that the exploration approaches proposed in these works are based on the {\em counting}  of state-action visitation, which is technically feasible only when the state space is discrete and finite. A group of works \cite{CMAE,LSVI_based} has developed exploration solutions for continuous-state environments, but they often lack theoretical guarantees, and do not necessarily apply to networked agents or require fully connected networks. \par

 \par 
 To this end, in this work, we propose a provably efficient exploration algorithm for multi-agent networks. Our model does not require full connectedness and assumes graph-based communication among the agents. For learning and execution, the agents do not need to observe the states and actions of all other agents nor share knowledge with all of them. This model  respects the privacy of agents and better reflects real word scenarios where the operation of agents is usually localized.  \par 
 Moreover, our proposed algorithm does not require a counting mechanism. This is a useful feature because counting is not feasible for large or continuous state-spaces, which is the case for many situations of interest.   To bypass this problem, we replace the operation of counting by the procedure of variance estimation within neighborhoods. In general, the use of estimated variances as an exploration signal is a known technique in the study of single-agent RL methods \cite{TDU, TDE}. In this scenario, the agent needs to generate an ensemble of estimates to estimate uncertainties. Extending the idea to the multi-agent case allows us to avoid the need for ensemble estimation at the individual agents. This is because agents can now tap into information from their immediate neighborhoods, thus reducing the workload for the agents. The challenge is to verify that by doing so, the resulting exploration strategy leads to effective learning. Indeed, we verify that the proposed method is  provably efficient in a discrete-state scenarios. \par 
 In the literature, to the best of our knowledge, the existing provably efficient MARL algorithms with exploration are those based on counting \cite{justin} and/or consider centralized architectures where agents communicate with all other agents\cite{LSVI_based}.  
In general, the addition of exploration mechanisms to the MARL setting introduces some stochastic components into the system's behavior, which make the theoretical analysis more demanding.  
Moreover,  a suitable definition of ``sufficient exploration'' remains open in the literature. One common metric for evaluating the exploration ability of a strategy is to show that it leads to a bounded regret. The regret, by definition,  is the expected reward loss due to following some behavioral policy instead of acting optimally. Therefore, when the regret is bounded, the behavioral policy is guaranteed to converge to the optimal one. However, regret-based analysis can be challenging, especially for scenarios with networked agents. One advantage of our exploration algorithm is that we can examine its convergence properties by  avoiding regret analysis.\par 
In particular, our algorithm is based on the paradigm of Q-learning, which is known to converge to the optimal solution under the assumption that all states and actions are visited infinitely often. In addition, recall that Q-learning is an off-policy algorithm, where an agent can learn the optimal (or target) policy based on the samples obtained from following a different policy. Such decoupling between target and behavioral policies allows us to design behavioral policies that ensure sufficient exploration during learning. Therefore, in our work, instead of pursuing a regret analysis, we design and verify that under the proposed behavioral policy all state and actions are visited infinitely often.
\par

  \par 
 \par
  Our work focuses on a specific scenario of MARL where agents exchange information to improve decision making but they are not collaborating in their actions.  
  In practice, this scenario fits search problems where  collaboration is not mandatory for fulfilling the target task. In such problems, it is not necessary that all agents come close to the target to declare that it was found. If one of the agents finds the target it can share this knowledge with others so that other agents can update their policies  and, at the next step, make decisions towards the target.  As a contrasting example consider a game where agents not only need to find the object but also to lift or push it. In this case, cooperation among agents in terms of actions is required since the heavy object cannot be moved or lifted unless all agents act simultaneously and in the same direction.   
   This type of multi-agent problems, where concepts of joint actions and joint rewards are omitted, are best described within the parallel MDP setting used in \cite{LSVI_based,justin}. More details are provided in Sec. \ref{sec:preliminaries}\par

\section{Related work}
\label{sec:literature}
A plethora of work, including  Independent Learning \cite{IL1}, Multi-Agent DDPG (MADDPG) \cite{MADDP},  Q-mix \cite{Q_mix} and VDN \cite{VDN} have been developed in the literature to capture nonstationarity of the environment and analyze other complexities in multi-agent problems \cite{IL1,MADDP,Q_mix,VDN}. However, these works focus on the exploitation part of RL problems assuming exploration granted. This assumption is problematic when the reward function depends only on a small subset of an exponentially growing state space.  When the observed reward samples are not immediately informative, then effective exploration becomes especially important. Therefore, in this work, we are interested in sparse reward environments, which most existing strategies in MARL do not contemplate. \par 
  One work that targets specifically sparse reward problems is Coordinated Multi-Agent Exploration (CMAE) \cite{CMAE}.   This work develops a state projection-based scheme, which allows agents to implement a from-low-to-high search in the state-space. CMAE considers a more general framework than ours as it allows rewards to depend on joint actions. However, in comparison, there are three issues that our work can address. First, CMAE is more suitable for grid-based environments. Its implementation in continuous-state games is complex and memory demanding due to the need for additional conversion techniques such as hash-tables and neural networks. Second, the CMAE framework is limited only to centralized training  and decentralized execution paradigm while our algorithm  operates in a {\em fully decentralized} mode. Finally, CMAE is mainly experimentally validated, with theoretical analysis provided for specific matrix games; its performance in a more general setting is not guaranteed.\par 
Two recent works that studied exploration in a parallel MDP setting are  a graph extension of the upper-confidence bound (GUCB) \cite{justin} and a multi-agent version of Least-square Value Iteration (MALSVI) \cite{LSVI_based}. Both algorithms are designed for an episodic MDP, which is commonly considered in less challenging scenarios than the one assumed in our model, where we consider a discounted infinite horizon MDP \cite{Inf_horizon,Inf_horizon2}. As the name implies, GUCB allows graph-based communication among agents. Due to the exploration bonus \cite{justin}, which depends on the structure of the graph and the number of state-action visitations, GUCB  executes a provably efficient exploration.  However, similar to the single agent UCB,  GUCB  is also limited only for the discrete state case and it requires counting. In comparison, MALSVI is  a multi-agent exploration algorithm that does not need state counting. However, it does not always allow agents to communicate in a distributed manner. Under MALSVI, agents learn independently and at some time instances, when a certain condition is met,  knowledge from all agents is aggregated. 
In other words, the operation of MALSVI requires periodical full-connectivity of agents, while our algorithm allows agents to have graph-structured connectivity. Since GUCB and MALSVI are the closest to our work in the setting, our simulation results will be presented in comparison to these two algorithms. 
 
\section{Markov Decision Process (MDP)}
\label{sec:preliminaries}
As an underlying model, we consider a scenario where $K$ agents are learning over $K$ identical parallel MDPs. 
 The agents can communicate and exchange information with their neighbors, as defined by some graph $G$. We let $\mathcal{E}$ denote the edge set of $G$ and introduce the index set $\mathcal{K}\triangleq [1,2 \ldots K]$. The neighbourhood of an agent $k$ is defined as $\mathcal{N}_k\triangleq\{j \in \mathcal{K}: (j,k) \in \mathcal{E}\}$. \par 
Each of the $K$ MDPs is defined
by a $5$-tuple $(\mathcal{S}, \mathcal{A}, \mathcal{P}, \mathcal{R},\gamma)$, where $\mathcal{S}$ is the state space, $\mathcal{A}$ is the action space,  $\mathcal{R}:\mathcal{S} \times \mathcal{A} \to \mathbb{R}$ is the reward function, and $\gamma\in(0,1)$ is a reward discount factor useful for infinite-horizon problems. For analysis purposes, the  state and action spaces are assumed to be discrete. However, in the implementation and computer simulations, we will illustrate the case in which the state space is continuous. Moreover, $\mathcal{P}:\mathcal{S} \times \mathcal{A} \to \mathcal{S}$ denotes the transition model, where the value $\mathcal{P}(s'|s,a)$ denotes the probability of transitioning from state $s$ to state $s'$ after taking action $a$.\par 
Let $\pi(a|s)$ denote an arbitrary policy, which describes the likelihood of selecting action $a$ at state $s$. Then, the state-action value associated with an arbitrary pair $(s,a)$ is defined by
\begin{align}
\label{eq:Q}
    Q^{\pi}(s, a)&\triangleq \mathbb{E}\left[\sum_{n=0}^{\infty} \gamma^n r_n \mid s_{0}=s, a_0=a\right] ,
\end{align}
where the shorthand notation $r_n\triangleq r(s_n,a_n,s_{n+1})$ represents the reward that results from taking action $a_n$ at state $s_n$, both at time $n$, and transitioning to state $s_{n+1}$. For simplicity of notation, we denote the cardinalities of the respective sets by the symbols $A\triangleq|\mathcal{A}|$, $S\triangleq|\mathcal{S}|$, and  $N_k\triangleq|\mathcal{N}_k|$.  \par 

\subsection{Single agent Q-learning}
The framework proposed in this paper will build upon Q-learning, which is  a popular model-free, online, and off-policy RL scheme \cite{Barto,sayed_2022}. The method is model-free because its update rule does not depend on using either the transition or reward functions of the MDP.  The method is online because the update rule depends on the current observations only, without the need for data accumulation. More specifically, let us assume that at time instant $n$ an agent observes a transition from state $s$ to state $s'$ under some action $a$ and receives the reward $r(s,a,s')$. Then, the state-action value estimates $\widehat{Q}_{n+1}(s,a)$ are updated as follows \cite{Barto,sayed_2022}:
\begin{align}
\label{eq:Q_update}
\widehat{Q}_{n+1}(s,a)=& \ \widehat{Q}_{n}(s,a) \nonumber
\\&+\alpha_n(s,a)\delta_n(s,a) \mathbb{I}(s=s_n \text{ and }a=a_n),
\end{align}
where $\alpha_n(s,a)$ is the learning rate, $$\delta_n(s,a)=r_n+\gamma \max \limits_{b \in \mathcal{A}}\widehat{Q}_n(s',b)-\widehat{Q}_n(s,a),$$ and $\mathbb{I}(\cdot)$ is the indicator function that is defined as 

\[ \mathbb{I}(\mathcal{B})=
   \begin{cases} 
      1, & \text{if }\mathcal{B} \text{ is true}\\
      0, & \text{ otherwise} 
   \end{cases}
\]
 The objective of Q-learning is to use the estimated state-action values in order to determine an approximate optimal policy $\pi^{\star}$ that maximizes (\ref{eq:Q}). 
 Although, the objective is to estimate $Q^{\pi^{\star}}$ under the optimal policy, one of the main features of Q-learning is that data can be collected under any other policy $\eta$,  called a {\em behavioral policy}. This off-policy nature  of Q-learning is convenient for developing an exploration strategy.   \par 
Now, most convergence proofs for Q-learning are based on the assumption that all states and actions are visited infinitely often \cite{Watkins,SA}.  In other words, when the behavioral policies do not guarantee sufficient exploration of the state-action space, then Q-learning will not necessarily converge. In this work, we will design a behavioral policy that ensures that all states and actions are visited infinitely often by leveraging cooperation among networked agents. 



\section{Exploration Strategy}
\label{sec:exploration}

The proposed multi-agent algorithm will involve a distributed realization of $K$ Q-learning schemes. Since we will be dealing with operations at multiple agents, we will now need to attach an agent index $k$ to the main variables, such as writing $r_n \to r_{k,n}$, $\widehat{Q}_{n} \to \widehat{Q}_{k,n}$.
We denote the history of transitions and rewards   experienced by agent $k$ up to time $n+1$  by \begin{equation}
    \mathcal{H}_{k,n}\triangleq \left \{s_{k,i}, a_{k,i},r_{k,i}\right\}_{i=0}^{n}.
\end{equation}
In the same token, we denote the aggregate history in the neighborhood of agent $k$ by the notation  
\begin{equation}
    \widetilde{\mathcal{H}}_{k,n}\triangleq \left\{\mathcal{H}_{\ell,n}\right\}_{\ell\in\mathcal{N}_k}.
\end{equation}
There is extensive discussion in the literature on the classical trade-off between exploration and exploitation of the state-action space. However, no clear measure of ``sufficient exploration'' exists. One popular measure is to rely on the use of regret bounds. In this paper, 
we deviate from this traditional view and establish ``sufficiency" of exploration by relying on the ability of the behavioral policy to ensure ``infinitely often visitations of all states and actions". To this end,  we will first introduce the design of the behavioral policy.  \par
The behavioral policy at every agent determines how the actions are selected by that agent as it traverses the state space. In this work, we model the behavioral policy  at agent $k$ and time $n$ with the following Boltzman distribution:
\begin{equation}
\label{eq:policy}
\eta_{k,n} \left(a|s\right)=\frac{\exp \left(\beta_{k,n}(s)\widetilde{Q}_{k,n}(s,a)\right)}{\sum_{b \in \mathcal{A}} \exp \left(\beta_{k,n}(s) \widetilde{Q}_{k,n}(s,b)\right)},
\end{equation}
where the notation $\widetilde{Q}(s,a)$ refers to an adjusted version of the estimated $Q$-value, namely, 
 $$\widetilde{Q}(s,a)\triangleq\sigma_{k,n} (s,a)+\widehat{Q}_{k,n}(s,a).$$ The bonus term  $\sigma_{k,n} (s,a)$ is used to motivate agent $k$ to explore not only rewarding states, but also states that have been visited less frequently. Obviously, we want $\sigma_{k,n}$ to be inversely related to the state-action visitation number: the less the state-action pair $(s,a)$ is  visited, the larger the value of $\sigma_{k,n}$ will be and the more the agent will be motivated to revisit this location. However, we will require this feature to be implemented without actual counting.
 
 For this purpose, we will compute  $\sigma_{k,n}(s,a)$ as a sample standard deviation over the Q-value estimates available to agent $k$ from its neighbours, i.e., $\{\widehat{Q}_{\ell,n}(s,a),\ \ell \in\mathcal{N}_k\}$. More formally,  
 
\begin{align}
\label{eq:std}
\sigma_{k,n}^2 (s,a)& \triangleq \frac{1}{N_k-1} \sum_{\ell\in \mathcal{N}_k}\left(\widehat{Q}_{\ell,n}(s,a)-\overline{Q}_{k,n}(s,a)\right)^{2}, 
\end{align}
where 
\begin{align}
    \overline{Q}_{k,n}(s,a)& \triangleq \frac{1}{N_k}\sum\limits_{\ell\in \mathcal{N}_k}\widehat{Q}_{\ell, n}(s,a).\nonumber
\end{align}
Intuitively, when the state-action pair $(s,a)$ is insufficiently explored by agent $k$, i.e., the visitation number is small or zero, the state-action value estimates $\{\widehat{Q}_{\ell,n},\ \ell\in \mathcal{N}_k\}$ will poorly agree with each other resulting in a large standard deviation $\sigma_{k,n} (s,a)$.  
Repeated visitations drive the state-action values within the neighborhood closer to each other and, therefore, the bonus term $\sigma_{k,n}(s,a)$ will become smaller. 
The form of the exploration bonus in (\ref{eq:std}) is inspired from the single-agent case \cite{TDU}.

\par 
The parameter $\beta_{k,n}$, which controls  the shape of the Boltzman distribution in \eqref{eq:policy}, is explicitly defined as 
 \begin{equation}
\label{eq:beta}
  \beta_{k,n}(s) \triangleq \frac{1}{D(s)}\ln \log_{\alpha}\left[\frac{(\sigma_q)^{2A}(N_k)^{-A}}{\prod \limits_{b\in \mathcal{A}} \mathbb{E} \left[\sigma_{k,n}^2(s,b)|\widetilde{\mathcal{H}}_{k,n}\right ]}\right]^{\oplus},
\end{equation} 
where 
\begin{equation}
   \sigma_q\triangleq\mathbb{V}(\widehat{Q}_{k,0}), \quad D(s) \triangleq \max \limits_{a,b\in \mathcal{A}}\left|\widetilde{Q}_{t}\left(s,a\right)-\widetilde{Q}_{t}(s, b)\right|,
\end{equation}
$0 < \alpha \leq 1/4$, and the operator $[\cdot]^{\oplus}$ is defined by:
    \[ [\psi]^{\oplus}=
   \begin{cases} 
      \psi, & \text{if }\psi\geq \alpha\\
      \alpha, & \text{ otherwise}. 
   \end{cases}
\]
We remark that in \eqref{eq:beta} and all further equations, unless otherwise stated, we will assume statistical operations such as \textit{expectation} ($\mathbb{E}$), \textit{variance} ($\mathbb{V}$), and \textit{covariance} ($\mathbb{C}\mathrm{ov}$) are, by default,  applied with respect to the initial state-action value estimates available at the neighbourhood $\mathcal{N}_k$, i.e., $\left\{\widehat{Q}_{\ell,0}(s,a),\ell\in \mathcal{N}_k \right\}, \forall s\in \mathcal{S}, a\in \mathcal{A}$.\par 

The above choice of $\beta_{k,n}$ allows us to guarantee the convergence of the proposed algorithm, as will become clear when we discuss Theorem 1. Note that the behavioral policies of the agents are based solely on local information from their neighborhoods. This fact and the assumption that actions are  decoupled imply that the exploration process in our algorithm takes place in a fully decentralized manner. \par
 Due to the expectation operator, we cannot implement (\ref{eq:beta}) in its exact form. To this end, we use the unbiased approximation:
 \begin{equation}
 \label{eq:approximation}
     \mathbb{E} \left[\sigma_{k,n}^2(s,a)|\widetilde{\mathcal{H}}_{k,n}\right ]\approx \sigma_{k,n}^2(s,a)
 \end{equation}
 whose accuracy is evaluated in Lemma \ref{lemma22:app_error}. Therefore, we arrive at the following listing for the Graph Exploration Algorithm (GEA) under Q-learning.

\begin{algorithm}
\SetAlgoLined
{\small
Initialize $\widehat{Q}_{k,0}(s,a)\sim \mathcal{M}$, $\forall k \in \mathcal{K}, s\in \mathcal{S}, a \in {A}$\;
\For{n=0,1,2....}{

  \For{all $k\in \mathcal{K
  }$}{
   observe own and neighbours' states $\{s_{\ell,n}\}_{\ell\in \mathcal{N}_{k}}$\;
   receive $\widehat{Q}_{\ell,n}(s_{k,n},a_{k,n})$ from $\ell \in\mathcal{N}_{k}$\;
   compute $\sigma_{k,n}(s_{k,n},a)$, $\forall a\in \mathcal{A}$ using (\ref{eq:std})\;
   take action $a_{k,n} \sim  \eta_{k,n} \left(a| s_{k,n}\right)$, given by               (\ref{eq:policy})\;
observe $r_{k,n}$ and $s_{k,n+1}$\;
 compute $\delta_{k,n}=r_{k,n}+ \max \limits_{b\in \mathcal{A}}\widehat{Q}_{k,n}(s_{k,n+1},b)-\widehat{Q}_{k,n}(s_{k,n},a_{k,n})$\;
update  $\widehat{Q}_{k,n+1}(s_{k,n},a_{k,n})=\widehat{Q}_{k,n}(s_{k,n},a_{k,n})+\alpha_{k,n}(s_{k,n},a_{k,n}) \delta_{k,n} $\;  
}

 }
 \caption{Multi-agent $Q$-learning with cooperative exploration for discrete state spaces (GEA).}
 }
 \label{Algorithm1}
\end{algorithm}

\subsection{ Asymptotic behavior of $\eta_{k,n}$}
In the sequel, we provide some well-known characteristics of the Boltzman distribution that could be useful for further analysis and implementation. For large $\beta_{k,n}$, the Boltzman distribution is equivalent to the greedy policy \cite{SARSA}, which can be expressed as:
\begin{equation}
\label{eq:asympt_eta1}
    \lim\limits_{\beta_{k,n}\to\infty}\eta_{k,n}=\arg\max\limits_{a\in \mathcal{A}} \widetilde{Q}_{k,n}(s_{k,n},a)
\end{equation}
 From \eqref{eq:beta}, it can be observed that $\beta_{k,n}$  approaches infinity when the estimate uncertainties tend to zero. Since the uncertainties are expected to vanish with the evolution of Q-learning, the limit in \eqref{eq:asympt_eta1} can be rewritten as:   
\begin{equation}
   \label{eq:asympt_eta2} \eta_{k,\infty}\triangleq\lim\limits_{n\to\infty}\eta_{k,n}=\arg\max\limits_{a\in \mathcal{A}} \lim\limits_{n\to\infty}\widetilde{Q}_{k,n}(s_{k,n},a)
\end{equation}
In contrast, when $\beta_{k,n}$ is zero, actions are selected  uniformly at random.
\section{Analysis and Discussions}
Most exploration algorithms lack theoretical guarantees, and the ones with theoretical results mainly rely on counting approaches. Theorem \ref{theorem1} states our first main contribution, where we guarantee convergence to the optimal policy under the proposed exploration mechanism.  First, we list some common assumptions.
\begin{assumption}[\textbf{Learning rate}]
\label{assumption:step}
Learning step sizes $\alpha_{k,n}(s,a)$,  for all states $s\in\mathcal{S}$, actions $a\in\mathcal{A}$, and agents  $k\in \mathcal{K}$, satisfy
\begin{align}
    &\sum \limits_{n=0}^\infty \alpha_{k,n}(s,a)=\infty,\nonumber\quad \quad \sum \limits_{n=0}^\infty \alpha_{k,n}^2(s,a)<\infty.\nonumber
\end{align} 
\end{assumption}
\begin{assumption}[\textbf{Initialization of Q-values}]
\label{assumption:initial}
    Q-value estimates for all agents $k\in \mathcal{K}$  are initialized  independently according to some distribution $\mathcal{M}$ with bounded support, i.e., $Q_{k,0}(s,a)<\infty$, and $\mathbb{E}\left[Q_{k,0}(s,a)\right]=0$, $s\in\mathcal{S}, a\in\mathcal{A}$. \par 
\end{assumption}

\begin{assumption}[\textbf{Transition probability}]
\label{assumption:trans}
 For any two subsets $\mathcal{S}_1\in \mathcal{S}$  and $\mathcal{S}_2\in\mathcal{S}$, there exist states $s_1\in \mathcal{S}_1$, $s_2\in \mathcal{S}_2$ and action $a \in \mathcal{A}$ such that $\mathcal{P}(s_2|s_1,a)>0$. \par  
 
\end{assumption}
\begin{assumption}[\textbf{Rewards}]
\label{assumption:reward}
 We assume that rewards are bounded, say, as   $r_{\min } \leq R\left(s, a, s^{\prime}\right) \leq r_{\max }, \forall s,s'\in\mathcal{S}$ and $ \forall a\in \mathcal{A}$. \par
 \QEDA
\end{assumption} 

Assumption \ref{assumption:trans} implies that any state in the MDP can be reached from any other state in a finite number of transitions. This assumption is required to ensure that all states are visited infinitely often.  

\begin{restatable}[\textbf{Convergence to the optimal policy}]{thm}{Ainur}
\label{theorem1}
Consider a discrete MDP with finite action and state spaces. Under Assumptions $1$-$4$, and under the exact construction (\ref{eq:beta}) for $\beta_{k,n}(s)$, the graph exploration algorithm   converges to the optimal policy $\pi^{\star}$ that maximizes \eqref{eq:Q}.

\end{restatable}
\begin{proof}
See Appendix \ref{section:theorem1}. In the proof, we extend arguments used for traditional Q-learning, but omit the common assumption that all states and actions are visited infinitely often. This is because the argument in the appendix shows that this property is guaranteed by the proposed exploration policy. \par 
\end{proof}
As already mentioned in \eqref{eq:approximation}, the implementation of the  algorithm replaces the expectation appearing in \eqref{eq:beta} by the sample approximation $\sigma_{k,n}^2$.  
 When $\widetilde{H}_{k,n}$ is given, randomness in $\sigma_{k,n}^2$ is only due to the random initialization of the Q-estimates. Therefore, with the evolution of the Q-learning algorithm, the effect of the random initializations diminishes bringing the approximation and the exact value closer to each other. Lemma \ref{lemma22:app_error} provides the asymptotic rate of convergence for the approximation error.     
\begin{restatable}[\textbf{Assymptotic approximation error}]{lemma}{Ainurl}
\label{lemma22:app_error}
 Let $p_{k}(s,a)$ denote 
state-action visitation probabilities induced by the asymptotic behavioral policy $\eta_{k,\infty}(a|s)$ \eqref{eq:asympt_eta2}, defined in \eqref{eq:asympt_eta2}. Also let  
\begin{align}
 p_{\min}&\triangleq \min\limits_{k\in \mathcal{K},s\in\mathcal{S},a\in\mathcal{A}}p_k(s,a),\\
 p_{\max}&\triangleq \max\limits_{k\in \mathcal{K},s\in\mathcal{S},a\in\mathcal{A}}p_k(s,a).
\end{align}
 Then, for all states $s\in\mathcal{S}$, actions $a\in\mathcal{A}$, agents $k\in\mathcal{K}$ and $n\to \infty$, with at least probability $(1-p_{\epsilon})$, it holds that
\begin{equation}
\label{eq:finite_time}
\Bigl|\sigma_{k,n}^2-\mathbb{E}\left (\left.\sigma_{k,n}^2\right|\widetilde{\mathcal{H}}_{k,n}\right)\Bigr| \leq \frac{B'N_k}{(N_k-1)p_{\epsilon}}\left(\frac{1}{n}\right)^{\frac{2p_{\min}}{p_{\max}}\gamma},
\end{equation}
 where  $0<B'<\infty$. 
\end{restatable}
\begin{proof}
    See Appendix \ref{sec:approximation}.
\end{proof}
Note that the bound in \eqref{eq:finite_time} depends on the probabilities of visitations $p_{k}(s,a)$, which are dictated by the behavioral policies $\eta_{k,\infty}(a|s)$. Since the proposed behavioral policy is designed to ensure infinitely often visitation of all states and actions it is reasonable to assume that  the $p_k(s,a)$ values are strictly positive. \par  
The bound in \eqref{eq:finite_time} is asymptotic since a finite-time analysis of standard Q-learning is in itself a challenging problem \cite{Finite_q_learning, Finite_time_q2}.  For environments that are highly sensitive to the approximation error and require that the error on the left-hand side of (\eqref{eq:finite_time}) be well-bounded in finite time,  a possible solution is averaging over multiple realizations. Namely, agents can learn and share multiple Q-estimates  using Q-learning bootstrapping methods\cite{TDU, Bootstrapping1, Bootstrapping2} and use these estimates to generate multiple realizations of $\sigma_{k,n}^2$, which, in turn, are averaged to have a better approximation for $\mathbb{E}(\sigma_{k,n}^2|\widetilde{\mathcal{H}}_{k,n})$. Then, the approximation error will scale inversely proportional to the number of per agent Q-estimates, by \textit{central limit theorem}. \par 

\begin{figure*}
\begin{minipage}[b]{.48\linewidth}
  \centering
  \centerline{\includegraphics[width=7.0cm]{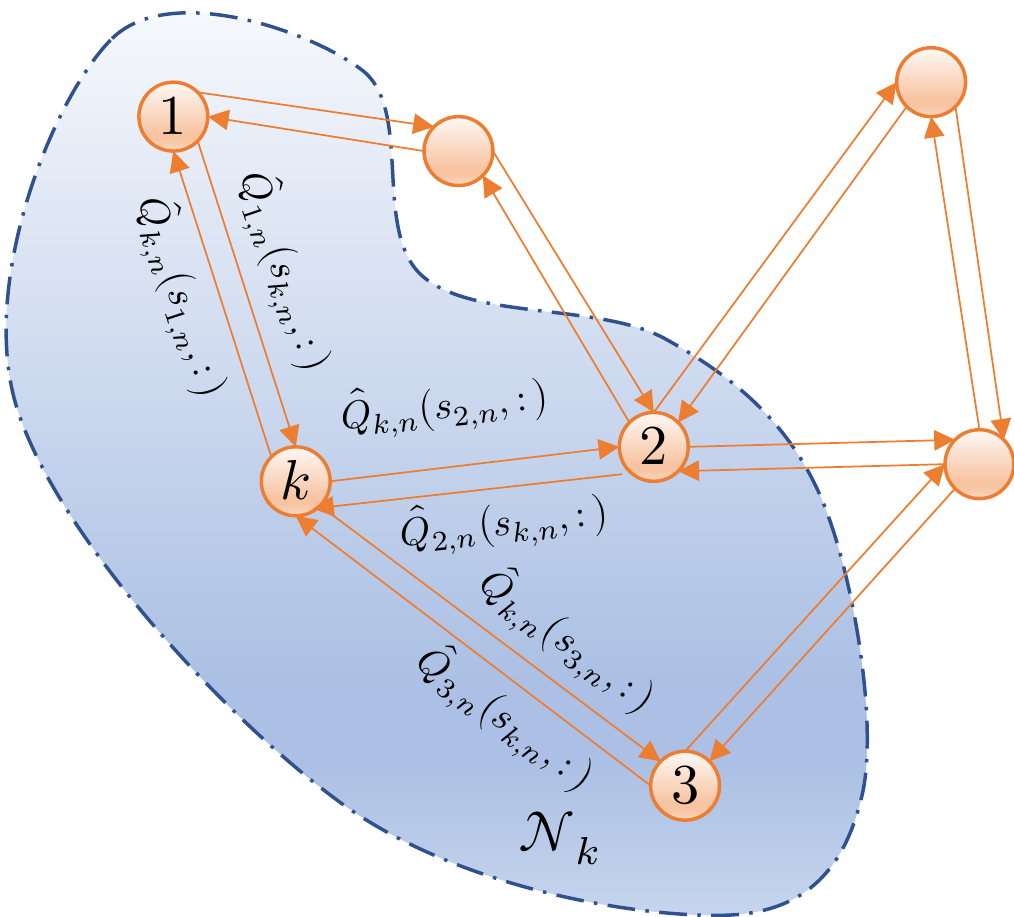}}
  \centerline{(a) Discrete states}\medskip
\end{minipage}
\hfill
\begin{minipage}[b]{0.48\linewidth}
  \centering
  \centerline{\includegraphics[width=7.0cm]{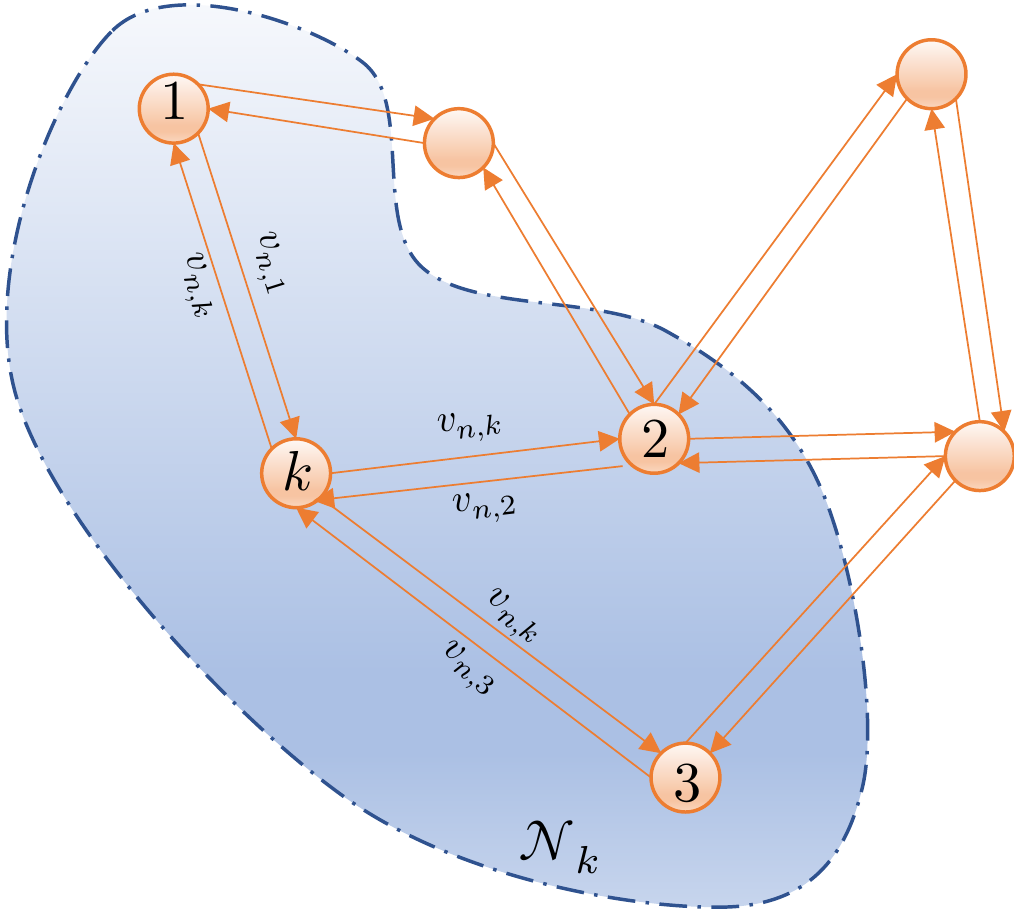}}
  \centerline{(b) Continuous states}\medskip
   
\end{minipage}
\caption{Diagram showing the information exchanged with neighbors of agent $k$ considering a discrete-state  scenario (a) and a continuous-state scenario (b).}
\label{fig:model}
\end{figure*}
\subsection{Continuous states and function approximation}

Until now, all discussions  assumed that the state space $\mathcal{S}$ is countable. However, since the proposed graph exploration algorithm does not require counting, it is particularly suitable for {\em continuous} state spaces as well. In the sequel, we explain how the implementation of the graph exploration algorithm changes for continuous-state problems. \par  

For continuous states, a tabular Q-learning is not an option since agents cannot generally deal with an infinite number of state-action values corresponding to the infinite number of states. 
However,  we can assume that there exists a generalization model $Q_{v}$ that can approximate the relation between any state-actions and their optimal state-action values, i.e.: 
$$Q^{\pi^{\star}}(s,a)\approx Q_{v}(f(s,a); v^{\star}),$$ where $f(s,a)\in \mathbb{R}^d$ denotes a  feature vector representing the state-action $(s,a)$ and  $v^{\star}$ is some parameter vector for the approximation function. By learning only $v^{\star}$, the optimal state-action values can be rebuilt  according to  $Q_v$, for any states and actions, even if they are continuous. \par 
 Therefore, in Q-learning with function approximations, instead of the optimal Q-values, agents learn the corresponding optimal parameter vectors $v^{\star}$. The optimal parameter estimates  $v_{k,n}$  are updated similar to \eqref{eq:Q_update}, except that the learning step, in addition, is  multiplied by the gradient of the approximation function:
 $$v_{k,n+1}=v_{k,n}+\alpha_{k,n}(s_{k,n},a_{k,n}) \delta_{k,n} \nabla_{v}Q_{v}(f_{k,n};v_{k,n})$$
where  for simplicity of notation we let $f_{k,n}\triangleq f(s_{k,n},a_{k,n})$.  
These parameters are further shared with neighbours in order to compute the
 state-action value estimates using the approximation $\widehat{Q}_{\ell,n}(s,a)\approx Q_{v}(f(s,a);v_k)$, which, in turn, are used to compute exploration policies given in \eqref{eq:policy}. 


Apart from allowing continuous states, the use of the approximation model significantly unloads the communication process among agents. According to \eqref{eq:std} and \eqref{eq:beta}, to compute $\beta_{k,n}(s)$ for the observed state $s_{k,n}$, agent $k$  needs the estimates $\{\widehat{Q}_{\ell,k}(s_{k,n},a)\}_{a\in\mathcal{A}}$, from all its neighbours $\ell \in \mathcal{N}_{k}$. Therefore, for discrete states, agent $k$ receives in total $N_kA$ estimates. For continuous states, all estimates in $\{\widehat{Q}_{\ell,k}(s_{k,n},a)\}_{a\in\mathcal{A}}$ can be computed using a single parameter vector $v_{\ell,n}$, which reduces the total number of received estimates to $N_{k}$. For visual explanation, see the example in Fig. \ref{fig:model}\par 
Moreover, the use of approximations simplify the operational process. For the discrete-state case, an agent needs to observe the states of all its neighbours, use the Q-table to look for the estimates corresponding to the observed states, and send the values to the corresponding neighbours. Meanwhile, for the continuous-state case, an agent does not technically need to observe other agents since it can simply broadcast a single parameter $v_{k,n}$ to everyone nearby.


\begin{algorithm}

\SetAlgoLined
Initialize parameters $v_{k,0}$, $\forall k \in \mathcal{K}$\;
\For{n=0,1,2....}{

  \For{all $k\in \mathcal{K
  }$}{
   Observe own state $\{s_{i,n}\}_{i\in \mathcal{N}_{k}}$\;
   Receive $v_{j,n}$ from all neighbours $j \in \mathcal{N}_{k}$\;
   compute $\sigma_{k,n}(s_{k,n},a)$, $\forall a\in \mathcal{A}$ using (\ref{eq:std}) and the approximation $\widehat{Q}_{k,n}(s,a)\approx Q_v(f(s,a);v_{k,n})$\;
   take action $a_{k,n} \sim  \eta_{k,n} \left(a| s_{k,n}\right)$, given in               (\ref{eq:policy})\;
observe $r_{k,n}$ and $s_{k,n+1}$ \;
 compute $\delta_{k,n}=r_{k,n}+ \max \limits_{b\in \mathcal{A}}Q_v(f(s_{k,n+1},b);v_{k,n})-Q_v(f(s_{k,n},a_{k,n});v_{k,n})$\;
update  $v_{k,n+1}=v_{k,n}+\alpha_{k,n}(s_{k,n},a_{k,n}) \delta_{k,n} \nabla_{v}Q_{v}(f_{k,n};v_{k,n})$\;  
  }

 }
 \caption{Multi-agent $Q$-learning with cooperative exploration for continuous  state spaces (GEA)}
 \label{Algorithm2}
\end{algorithm}

Finally, it is important that we make the following remark regarding continuous state-space applications. Continuous state-space analysis is usually non-tractable due to the use of neural networks. If we decide to focus on the linear case for the sake of the theoretical analysis, there are actually not many realistic examples whose environment can be linearly modelled. Thus, by designing an algorithm based on the discrete state-space case, while targeting continuous state-space games, we are implicitly assuming that there is a correlation  between discrete and continuous cases. The algorithm, which converges under a discrete case analysis, should also work for continuous case if a good approximation model is chosen. The choice of approximation models and their evaluation is in itself a rich research topic, which deserves future study. In our work, we focus on providing experimental results that show the algorithm's good performance under continuous-state games.


\section{Experimental Results}
\begin{figure}[h]
\begin{minipage}[b]{1\linewidth}
  \centering
  \centerline{\includegraphics[width=7.0cm]{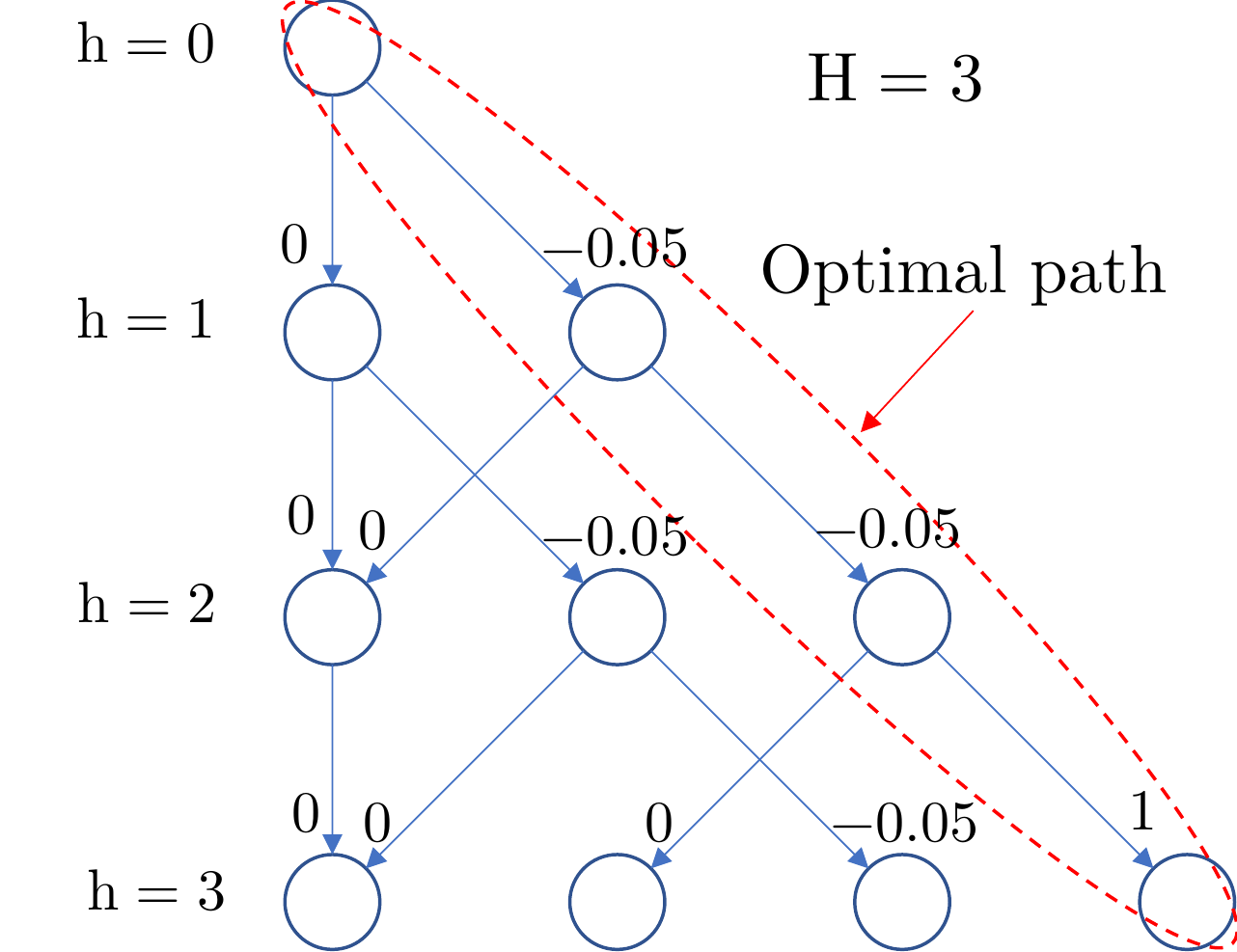}}
\end{minipage}
\caption{Diagram illustrating the deep sea game for a depth $H=3$ and the action space $\mathcal{A}=\{\text{RIGHT}, \text{LEFT}\}$. The arrows and numerical values above them illustrate the game's transition model and the reward function, respectively.}
\label{fig:deep_sea}
\end{figure}

\begin{figure*}[h]
\begin{minipage}[b]{0.5\linewidth}
  \centering
  \centerline{\includegraphics[width=10cm]{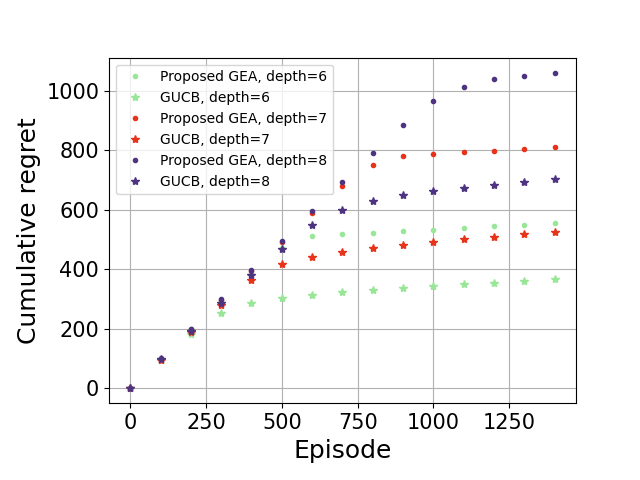}}
  \centerline{(a) GEA vs GUCB.  }\medskip
\end{minipage}
\begin{minipage}[b]{0.5\linewidth}
  \centering
  \centerline{\includegraphics[width=10cm]{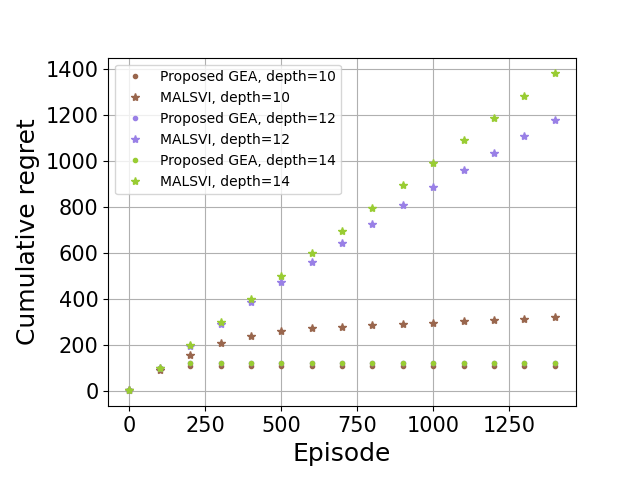}}
\centerline{(b) GEA vs MALSVI: original size.  }\medskip
\end{minipage}
\begin{minipage}[b]{0.5\linewidth}
  \centering
  \centerline{\includegraphics[width=10cm]{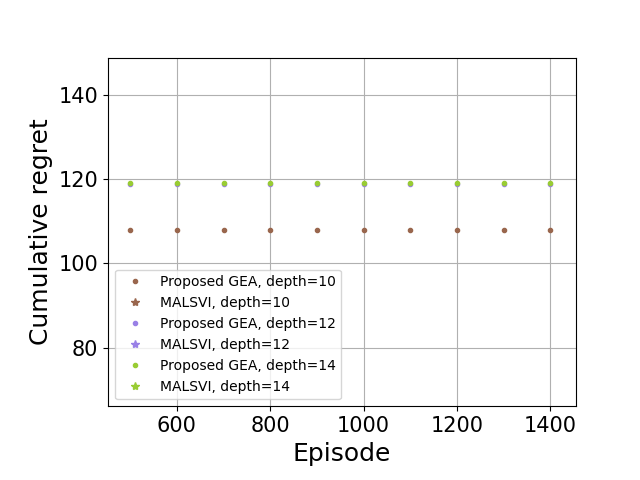}}
  \centerline{(c) GEA vs MALSVI: zoom 1. }\medskip
  
\end{minipage}
\begin{minipage}[b]{0.5\linewidth}
  \centering
  \centerline{\includegraphics[width=10cm]{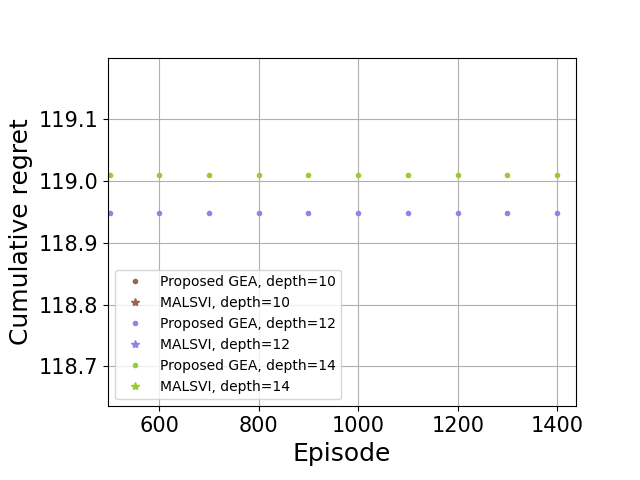}}
  \centerline{(d) GEA vs MALSVI: zoom 2. }\medskip
\end{minipage}
\caption{Regret values for several implementations, including the proposed GEA method, the discrete-state algorithm GUCB and the continuous-state algorithm MALSVI under changing sparsity and discrete states: (a) compares the proposed GEA and GUCB; (b)  compares the proposed GEA and MALSVI including results for all selected depth values;  (c) compares the proposed GEA and MALSVI zooming into regions where the performance of GEA is visible for $\operatorname{depth}=12$ and $\operatorname{depth}=14$; (d) compares the proposed GEA and MALSVI zooming into regions where the performance of GEA is visible for $\operatorname{depth}=10$ and $\operatorname{depth}=14$. }
\label{fig:res1}
\end{figure*}

Performance of  the proposed algorithm was tested on the \textit{deep sea} environment.  Deep sea, illustrated in Fig. \ref{fig:deep_sea},  is a challenging sparse reward environment where a positive reward can be achieved with one policy out of $2^H$ possible ones, where $H$ is the length of an episode. Moreover,  a deep sea player is deceived with  negative rewards for steps taken toward the optimal path. \par 

Following a common strategy in literature, the performance of the algorithm was tested based on the regret, which is formally defined in Definition \ref{definition:regret}. A bounded regret means that the behavioral policies of all agents can converge to the optimal one in finite time. The regret value at the convergence time shows how mistakenly agents have behaved before converging to the optimal solution. Also,  in our simulations, we vary the length of the game $H$ to see how the algorithm behaves with the changing sparsity. The relation between  sparsity of the environment and $H$ is direct: ``the deeper" ($H \uparrow$) the game is the further  the agents start  from the most rewarding state and the more paths with deceiving rewards will appear (sparsity $\uparrow$). 



\begin{definition}[\textbf{Regret expression}]
\label{definition:regret}The \textit{regret} at time instant $T$, for finite MDP with initial state $s_{k,0}$ for agent $k$, is given by  
\begin{equation}
\label{eq:regret}
\text{Regret}(T)=\frac{1}{K}\sum_{n=1}^{T}\sum_{k=1}^{K}\left(V^{\pi^\star}\left(s_{k,0}\right)-V^{\eta_{k,n}}\left(s_{k,0}\right)\right), 
\end{equation}
where $\pi^*$ is an optimal policy, which maximizes the expected total return of the game, $\eta_{k,n}$ is an actual policy followed by an agent at the $n$-th episode, and, under policy $\pi$, the state value  $V^{\pi}(s)$ is defined as 
\begin{align}
    V^{\pi}\left(s\right)\triangleq \mathbb{E}\left(\sum\limits_{n=0}^{\infty} \gamma^{n}r_{n}|s_0=s\right). \nonumber
\end{align} \QEDA
\end{definition}

In our illustrations, for better visualization, we duplicate the results by zooming over poorly-visible regions. Precisely, Fig \ref{fig:res1}b
is an original-size figure with regret plots for the proposed GEA and the MALSVI algorithm across all selected depth values, while Fig. \ref{fig:res1}c and Fig \ref{fig:res1}d zoom into regions where the performance of GEA is visible for specific depth values. In particular, Fig.\ref{fig:res1}c focuses on the performance of GEA for $\operatorname{depth}=12$ and $\operatorname{depth}=14$, while Fig.\ref{fig:res1}d focuses on the performance of GEA for $\operatorname{depth}=10$ and $\operatorname{depth}=14$.  \par 
Our algorithm is first compared with the GUCB approach, which is designed for discrete-state environments only. GUCB is limited for discrete-state cases only because its exploration bonus depends on the counting of state-action visitation frequencies:
\begin{equation}
\label{eq:b_GUCB}
b^{\text{\tiny{GUCB}}}_{k,n}(s,a)\triangleq\beta \sqrt{\frac{H^3\iota } {w_k c_{k,n}(s,a)}}
\end{equation}
where $\beta>0$ is a constant, $c_{k,n}(s,a)$ denotes the visitation frequency of state-action $(s,a)$ by agent $k$ before time $n$, $w_k$ is a fixed parameter that depends on the structure of the graph, and $\iota$ is a fixed parameter that depends on the properties of the MDP. From Fig. \ref{fig:res1}a, we can see that GUCB outperforms the proposed GEA algorithm in terms of total regret for all depth values. This outcome is anticipated, as counting confers an advantage to algorithms in accurately quantifying the uncertainty associated with states and actions. However, counting-based approaches limit the states to be discrete, which reduces the range of applicable problems. Therefore, we posit that the incurred disadvantage in terms of total regret is a reasonable compromise for adopting a non-counting-based approach in the proposed algorithm. \par 
Furthermore, it is imperative to not only assess the point of convergence of regret, but also the rate of convergence. Specifically, while an approach may achieve convergence to the optimal solution by exploring various sub-optimal paths, the time taken to achieve such convergence holds considerable significance. Given that the proposed bonus exhibits a higher level of stochasticity in comparison to traditional count-based mechanisms, it may potentially result in the selection of paths with elevated regret values. However, in terms of convergence time, our algorithm can be evaluated as relatively competitive when compared to the GUCB approach.

 Our algorithm is also compared with MALSVI, which allows states to be continuous under the assumption that the model is linear: $$Q^{\pi^{\star}}\approx f^T(s,a) v^{\star}.$$ MALSVI belongs to a different class of RL algorithms, where, instead of Q-learning, the \textit{Least-Squares Value Iteration (LSVI)} method is employed. In LSVI approaches, optimal state-action values are estimated by solving iteratively a certain regularized least-squares regression. The exploration bonus in MALSVI is not directly related to the state-action visitation frequencies. It is rather designed with intention to overestimate Q-values.  This technique, called \textit{optimism in the face of uncertainty}, is a common strategy for bounding the regret analytically. The regret-based analysis in MALSVI guarantees convergence of the algorithm to the optimal solution but does not ensure that it will be reached in the most efficient way.  That is why the experimental results illustrated in Fig. \ref{fig:res1} show that the proposed scheme outperforms MALSVI, in terms of the regret. Moreover, as it was mentioned earlier, the operation of MALSVI requires that all agents are fully connected at certain instances. In particular,  agents individually solve their own local LSVI and when a certain condition is met  they need to synchronize their results. Thus, MALSVI assumes more restricted network structure compared to the proposed algorithm. 
 
 



\section{Conclusion}
To sum up, we propose an MARL exploration strategy that can be used for both continuous and discrete-state environments.  We provide theoretical guarantees for discrete-state scenarios, and simulation results for the continuous-state ones.  We consider a specific model of multi-agent learning where agents receive only individual rewards independent from actions of others. Therefore, the first extension for this work would be to consider MARL scenarios with joint actions and rewards. Second, one can extend the work by finding better approximations for the conditional expectation of sample variances. It is also useful to study the scenario with partially observable states and employ cooperation of agents for estimation of states as well.

\begin{appendices}
\section{Proof of Theorem 1}
\label{section:theorem1}
\begin{lemma}
\label{lemma:u_bound}
Under Assumption \ref{assumption:initial},  the variance of $Q_{k,n}(s,a)$, conditioned on the history $\mathcal{H}_{k,n}$,   is bounded as
\begin{equation}
 \mathbb{V}(Q_{k,n}(s,a)|\mathcal{H}_{k,n})\geq (\alpha)^{2c_{k,n}(s,a)}  \sigma_q^2, \ \forall k\in \mathcal{K},
\end{equation}
where $0 < \alpha \leq 1/4$ and $c_{k,n}(s,a)$ is the number of times a state-action pair $(s,a)$ is visited by agent $k$ by the time $n$.
\end{lemma}
\begin{proof}
 Let $\alpha_i\triangleq\alpha_{n_i, k}(s,a)$ denote  the step size at time instant $n_i\triangleq n_i(s,a,k)$ when agent $k$  visits the state-action pair $(s,a)$  for the $i$-th  time. For compactness, from here on, we 
 interchangeably use the following notations:
 \begin{align}
     t\triangleq c_{k,n}(s,a), \quad \quad  \widehat{b}_{k,n}\triangleq \arg\max \limits_{b\in \mathcal{A}}\widehat{Q}_{k,n}(s_{k,n+1}, b),
 \end{align}
 \begin{align}
     x_{k,n}\triangleq(s_{k,n},a_{k,n}), \quad y_{k,n}\triangleq(s_{k,n},\widehat{b}_{k,n}). 
 \end{align}
 The Q-learning update rule in (\ref{eq:Q_update}) leads to
 \begin{align}
 \label{eq111}
     &\widehat{Q}_{k,n}(x_{k,n})=\alpha_{0}^{t} \widehat{Q}_{k,0}(x_{k,n}) \nonumber\\
     &+\sum_{i=1}^{t} \alpha_{t}^{i}\left(r(x_{k,n_i},s_{k,n_i+1})+\gamma \widehat{Q}_{k,n}(y_{k,n_i})\right), 
 \end{align}
 where 
 \begin{equation}
\alpha_t^0=\prod_{j=1}^t\left(1-\alpha_j\right), \quad \alpha_t^i=\alpha_i \prod_{j=i+1}^t\left(1-\alpha_j\right).
\end{equation}
Therefore, at the next steps, we will lowerbound $\mathbb{V}(Q_{k,n}(s,a)|\mathcal{H}_{k,n})$ by applying the conditional variance to the left-hand side of \eqref{eq111}. First, using  (\ref{eq111}) and the variance property $\mathbb{V}(X+Y)=\mathbb{V}(X)+\mathbb{V}(Y)+2\operatorname{Cov}(X,Y)$, we have
\begin{align}
\label{equation2}
    &\mathbb{V}\left(Q_{k,n}\vert\mathcal{H}_{k,n}\right)= \mathbb{V}\left [\alpha_{0}^{t} \widehat{Q}_{k,0}(x_{k,n})\right.\nonumber \\ 
    &\left.\left.+\sum_{i=1}^{t} \alpha_{t}^{i}\left(r(x_{k,n_i},s_{k,n_i})+\widehat{Q}_{k,n}(y_{k,n_i})\right)\right| \mathcal{H}_{k,n}\right] \nonumber \\
    &=(\alpha_0^t)^2\sigma_q^2 \nonumber
    \\&+\sum_{i=1}^{t} (\alpha_{t}^{i})^2  \mathbb{V}\left.\left(r(x_{k,n_i},s_{k,n_i})+\gamma^2\widehat{Q}_{k,n}(y_{k,n_i})\right | \mathcal{H}_{k,n}\right) \nonumber\\
    &+2\sum_{i=1}^{t} \alpha_{t}^{i}\gamma \operatorname{Cov}\left. \left(\alpha_{0}^{t} \widehat{Q}_{k,0}(x_{k,n}),\widehat{Q}_{k,n}(y_{k,n_i}))\right | \mathcal{H}_{k,n}\right) \nonumber\\
    &+2\sum_{i=1}^{t} \alpha_{t}^{i}\operatorname{Cov}\left. \left(\alpha_{0}^{t} \widehat{Q}_{k,0}(x_{k,n}), r(x_{n_i},s_{k,n_i+1}))\right | \mathcal{H}_{k,n} \right).
\end{align}
The second term  in \eqref{equation2} is non-negative by definition of the variance. The third term can be expanded using the definition of covariance and lower-bounded as follows 
\begin{align}
\label{eq:inter1}
    &\operatorname{Cov}\left( \left. \widehat{Q}_{k,0}(x_{k,n}) \widehat{Q}_{k,n}(y_{k,n_i})\right|\mathcal{H}_{k,n}\right)\nonumber \\
    &= \mathbb{E}\left( \left.\widehat{Q}_{k,0}(x_{k,n})   \widehat{Q}_{k,n}(y_{k,n_i})\right|\mathcal{H}_{k,n}\right)\nonumber\\
    &- \mathbb{E}\left(\left. \widehat{Q}_{k,0}(x_{k,n})\right|\mathcal{H}_{k,n}\right)  \mathbb{E}\left( \left.\widehat{Q}(y_{k,n_i})\right|\mathcal{H}_{k,n}\right)\nonumber\\
    & \stackrel{\text{(a)}}{=}\mathbb{E}\left( \left.\widehat{Q}_{k,0}(x_{k,n})   \widehat{Q}_{k,n}(y_{k,n_i})\right|\mathcal{H}_{k,n}\right)\stackrel{\text{(b)}}{\geq} 0
\end{align}
where in (a) we use the assumption $\mathbb{E}(\widehat{Q}_{k,0})=0$ (Assumption \ref{assumption:initial}) and (b) holds by Lemma \ref{lemma4} bellow. \par 
 As initialization of Q-values are assumed to be independent from the reward function, the last term in (\ref{equation2}) is zero.

Therefore, expression \eqref{equation2}  can be lower bounded by $(\alpha_0^t\sigma_q)^2$.
\end{proof}
\begin{lemma}
\label{lemma4}
For all states $s,s'  \in \mathcal{S}$, all actions $ a,a' \in \mathcal{A}$, and all agents $k \in \mathcal{K}$, under Assumption \ref{assumption:initial}, the following holds
\begin{equation}
\label{equation3}
     \mathbb{E}\left(\widehat{Q}_{k,0}(s,a)  \widehat{Q}_{k,n}(s',a')\right) \geq 0
\end{equation}
\end{lemma}
\begin{proof}
We prove the lemma by induction. For $n=0$, $\forall s,s' \in \mathcal{S}, \forall a,a' \in \mathcal{A}$,
\begin{align}
    &\mathbb{E}\left(\widehat{Q}_{k,0}(s,a)  \widehat{Q}_{k,0}(s',a')\right)\nonumber\\
    &=\mathbb{E}\left(\widehat{Q}_{k,0}(s,a)^2\right)\mathbbm{1}[(s,a)=(s',a')]\nonumber\\
    &+\mathbb{E}\left(\widehat{Q}_{k,0}(s,a)\right)\mathbb{E}\left(\widehat{Q}_{k,0}(s',a')\right)\mathbbm{1}[(s,a)\neq(s',a')]\nonumber\\
    &=\sigma_q^2\mathbbm{1}\left[(s,a)=(s',a')\right]+0\geq 0, 
\end{align}
where in the last equality we used the assumption $\mathbb{E}\left(\widehat{Q}_{k,0}(s,a)\right)=0$ (Assumption \ref{assumption:initial}). Next, we assume that (\ref{equation3}) holds for $n=i$, i.e.,
\begin{equation}
\label{eq:induction}
     \mathbb{E}\left(\widehat{Q}_{k,0}(s,a)  \widehat{Q}_{k,i}(s',a')\right) \geq 0, \\ \forall s,s' \in \mathcal{S}, \forall a,a' \in \mathcal{A}.
\end{equation}
Then, for $n=i+1$,  we have
\begin{align}
   & \mathbb{E}\left(\widehat{Q}_{k,0}(s,a)  \widehat{Q}_{i+1,k}(s,a)\right)\stackrel{\text{(a)}}{=} \nonumber \\
   &(1-\alpha_{k,i+1}(s',a')) \mathbb{E}\left(\widehat{Q}_{k,0}(s,a)  \widehat{Q}_{k,i}(s',a') \right) \nonumber\\
   &+\alpha_{k,i+1}(s',a')\mathbb{E}\widehat{Q}_{k,0}(s,a)  \mathbb{E}r(s',a',s'') \nonumber\\
   &+\alpha_{k,i+1}(s',a')\gamma \mathbb{E}\left(\widehat{Q}_{k,0}(s,a)\max\limits_{b\in \mathcal{A}}\widehat{Q}_{k,i}(s'',b )\right) \nonumber\\
   &\stackrel{\text{(\ref{eq:induction})}}{\geq}  \alpha_{k,i+1}(s',a')\gamma\mathbb{E}\left(\widehat{Q}_{k,0}(s,a) \max\limits_{b\in \mathcal{A}}\widehat{Q}_{k,i}(s'',b )\right) \nonumber\\
   &\geq \alpha_{k,i+1}(s',a')\gamma   \mathbb{E}\left(\widehat{Q}_{k,0}(s,a)\widehat{Q}_{k,i}(s'',b'')\right)\nonumber \\
   &\stackrel{\text{\eqref{eq:induction}}}{\geq}0
\end{align}
where $b''\in\mathcal{A}$ and in (a) we use the multi-agent form of the Q-learning update given in \eqref{eq:Q_update}.
\end{proof}

\begin{lemma}
\label{lemma3}
For all states $s\in \mathcal{S}$, for all actions  $a\in \mathcal{A}$ and for all agents $k \in \mathcal{K}$, 
\begin{equation}
 \mathbb{E}(\sigma_{k,n}^2(s,a)|\widetilde{\mathcal{H}}_{k,n})\geq \frac{\alpha^{c_{k,n}(s,a)}}{N_k}  \sigma_q^2, \ \forall k\in \mathcal{K}
\end{equation}
\end{lemma}
\begin{proof}
 From \eqref{eq111}  we deduce that, when $\widetilde{\mathcal{H}}_{k,n}$ is given, $\left\{ \widehat{Q}_{\ell,n} \right\}_{j\in \mathcal{K}}$ are independent but not identically distributed random variables. Therefore, using the definition of expected sample variance for independent random variables, we get 
 \begin{align}
 \label{eq:inter222}
    & \mathbb{E}(\sigma_{k,n}^2(s,a)|\widetilde{\mathcal{H}}_{k,n})=\frac{1}{N_{k}}\sum\limits_{\ell\in \mathcal{N}_k}\mathbb{V}\left(\widehat{Q}_{\ell,n}(s,a)\left.\right|\mathcal{H}_{k,n}\right)\nonumber \\
    & +\frac{1}{N_k-1}\sum\limits_{\ell\in\mathcal{N}_k}\left[E(\widehat{Q}_{\ell,n}(s,a)|\mathcal{H}_{\ell,n})\right]^2\nonumber\\
    &-\frac{1}{N_k(N_k-1)}\left[\sum\limits_{j\in \mathcal{N}_{k}}\mathbb{E}\left[\left.\widehat{Q}_{\ell,n}(s,a)\right|\mathcal{H}_{\ell,n}\right)\right]^2.
 \end{align}
 By Cauchy-Schwartz inequality, we have 
\begin{align}
\label{eq:inter2}
    \left[\sum\limits_{j\in \mathcal{N}_{k}}\mathbb{E}\left(\left.\widehat{Q}_{\ell,n}\right|\mathcal{H}_{\ell,n}\right)\right]^2\leq N_k\sum\limits_{\ell\in\mathcal{N}_k}\left[E(\widehat{Q}_{\ell,n}|\mathcal{H}_{\ell,n})\right]^2.
\end{align}
Therefore, substituting \eqref{eq:inter2} into \eqref{eq:inter222}, we obtain that,   $\forall k \in \mathcal{K}$, $s\in\mathcal{S}$, $a\in\mathcal{A}$:
\begin{align}
    \mathbb{E}(\sigma_{k,n}^2(s,a)|\widetilde{\mathcal{H}}_{k,n})&\geq\frac{1}{N_{k}}\sum\limits_{\ell\in \mathcal{N}_k}\mathbb{V}\left(\widehat{Q}_{\ell,n}(s,a)\left.\right|\mathcal{H}_{\ell,n}\right)\nonumber \\
    &\stackrel{\text{(a)}}{\geq}\frac{1}{N_{k}}\sum\limits_{\ell\in \mathcal{N}_k} \alpha^{c_{k,n}(s,a)}  \sigma_q^2  \nonumber\\
    & \geq\frac{\alpha^{c_{k,n}(s,a)}}{N_k}  \sigma_q^2
\end{align}
where (a) is due to Lemma \ref{lemma:u_bound}.
\end{proof}

\begin{proof}[\textbf{Proof of Theorem 1}]
Update rules of the form \eqref{eq:Q_update} is known to converge to the optimal Q-values given that all state-action pairs are visited infinitely often (i.o.) \cite{sayed_2022, Watkins}. Therefore, to prove that all agents in the proposed scheme converge to the optimal Q-values we need to show that the behavioral policies $\eta_{k,n}$ make all agents to visit all state-action pairs i.o. \par 
According to Lemma 4 from \cite{SARSA}, all states and actions are visited i.o. if
\begin{enumerate}[ (i)]
    \item  Assumption \ref{assumption:trans} holds
    \item  The behavioral policy $\eta_{k,n}$ is such that, $\forall k\in \mathcal{K}, \forall s \in \mathcal{S}, \forall a \in \mathcal{A}$,  $$\sum \limits_{j=0}^{\infty}\eta_{k,m_{j}}(a|s)=\infty,$$ where $m_j\triangleq m_j(s,k)$ is the time instance when the state $s$ is visited by agent $k$ for the $j$-th time.
\end{enumerate}

Since (i) is taken as an assumption in the theorem,  the proof of Theorem 1 requires that we prove that (ii) holds. Recall that  our behavioral policy is given in the Boltzmann distribution form (3), namely, 
\begin{equation}
\label{eq:Boltzman}
    \eta_{k,n}\left(a| s\right)=\frac{\exp \left(\widetilde{\beta}_{k,n}(s)\widetilde{Q}_{k,n}(s,a)\right)}{\sum_{b \in \mathcal{A}} \exp \left(\widetilde{\beta}_{k,n}(s) \widetilde{Q}_{k,n}(s,b)\right)}. 
\end{equation}

We first derive the condition for the Boltzman probability density function with  generic $\widetilde{\beta}_{k,n}$ (\ref{eq:Boltzman}) in order to ensure that (ii) holds. We know that  $\sum\limits_{i=1}^{\infty}1/i=\infty$. Due to Assumption 3, the number of visitations of state $s$ by agent $k$, denoted by $c_{k,n}(s)$, goes to infinity (see proof of Lemma 4 in \cite{SARSA}).  Therefore, we can add the following constraint for the behavioral policy:
\begin{equation}
\label{eq:policy_constraint}
    \eta_{k,n}\left(a| s\right)=\frac{\exp \left(\widetilde{\beta}_{k,n}(s)\widetilde{Q}_{k,n}(s,a)\right)}{\sum_{b \in \mathcal{A}} \exp \left(\widetilde{\beta}_{k,n}(s) \widetilde{Q}_{k,n}(s,b)\right)} \geq \frac{C'}{c_{k,n}(s)},
\end{equation}
where $C'>0$ is a constant. In the following, using (\ref{eq:policy_constraint}) and some basic algebraic manipulations, we derive the condition for $\beta_{k,n}$ to satisfy (ii). 
We can  rewrite (\ref{eq:policy_constraint}) as 

\begin{align}
    c_{k,n}(s) & \geq C' \sum_{b \in A} e^{\widetilde{\beta}_{k,n}(s) (\widetilde{Q}_{k,n}(s, b)-\widetilde{Q}_{k,n}(s, a))}
\end{align}
Next, taking the maximum of the lowerbound and letting $C'=A^{-1}$ we obtain 
\begin{align}
c_{k,n}(s) & \geq e^{\widetilde{\beta}_{k,n}(s)D(s)}, \nonumber \\
\end{align}
where $D(s)=\max\limits_{a,b}(\widetilde{Q}(s,b)-\widetilde{Q}(s,a))$. Hence, we get that (ii) is satisfied if  $\widetilde{\beta}_{k,n}$ is bounded according to 

\begin{equation}
\label{eq:beta_condition}
\widetilde{\beta}_{k,n} \leq \frac{\ln  c_{k,n}(s)}{D(s)} 
\end{equation}
Next, we show that $\widetilde{\beta}_{k,n}=\beta_{k,n}$, given by (\ref{eq:beta}), satisfies condition (\ref{eq:beta_condition}). We start with Lemma \ref{lemma3} which gives  the lower bound for $ \mathbb{E}(\sigma_{k,n}^2|\mathcal{H}_{k,n})$, under Assumption \ref{assumption:initial}:
\begin{align}
\label{eq:inter}
    & \mathbb{E}(\sigma_{k,n}^2(s,a)|\widetilde{\mathcal{H}}_{k,n})\geq\frac{\alpha^{c_{k,n}(s,a)}  \sigma_q^2}{N_k} 
\end{align}
In the following steps, we perform some algebraic manipulations to  (\ref{eq:inter}): 
\begin{enumerate} [label=\ Step \arabic*.,align=left, leftmargin=*]
    \item Multiplication to $N_k$ and applying $\log_{\alpha}$: 
    \begin{align}
        \log \limits_{\alpha}N_k \mathbb{E}(\sigma_{k,n}^2|\widetilde{\mathcal{H}}_{k,n})\leq c_{k,n}(s,a)+\log \limits_{\alpha}\sigma_q^2  
    \end{align}
    \item Subtracting $\log \limits_{\alpha}\sigma_q^2$ and summing over all possible actions:
\begin{align}
        &\sum\limits_{a\in\mathcal{A}}\log \limits_{\alpha}\frac{N_k \mathbb{E}(\sigma_{k,n}^2(s,a)|\widetilde{\mathcal{H}}_{k,n})}{\sigma_q^2}\leq c_{k,n}(s) 
\end{align}
where by definition
\begin{equation}
    c_{k,n}(s)\triangleq \sum \limits_{a\in \mathcal{A}}c_{k,n}(s,a), 
\end{equation}
\item Applying $\ln$:
\begin{align}
\label{eq:444}
   \ln \left(\log_{\alpha}\frac{(\sigma_q)^{2A}(N_k)^{-A}}{\prod \limits_{b\in \mathcal{A}} \mathbb{E} \left[\sigma_{k,n}^2(s,b)|\widetilde{\mathcal{H}}_{k,n}\right ]}\right)\leq \ln c_{k,n}(s)
\end{align}
\end{enumerate}

Finally, dividing both sides of \eqref{eq:444} by $D(s)$, we show that, by choosing $\widetilde{\beta}_{k,n}=\beta_{k,n}$, we satisfy \eqref{eq:beta_condition} and moreover we satisfy (ii).

\end{proof}
\section{Proof of Lemma \ref{lemma22:app_error}}
\label{sec:approximation}
\begin{lemma} [\textbf{Asymptotic convergence rate of Q-learning} \cite{Q_rate}]
\label{lemma:Q_rate}
Let $$e_{k,n}(s,a)\triangleq\left (\widehat{Q}_{k,n}(s,a)-Q^*(s,a)\right).$$
 Then, for all agents $k\in\mathcal{K}$, all states $s \in \mathcal{S}$, all actions $a\in \mathcal{A}$ and $n\to \infty$, we have 
  \begin{equation}
\label{eq:error_ap}
\begin{aligned}
&e_{k,n}(s,a)\leq \left(\frac{1}{n}\right)^{\frac{p_{k,\min}}{p_{k,\max}}\gamma}B,
\end{aligned}
\end{equation}
where $p_{k}(s,a)$ denotes 
state-action visitation probabilities induced by some stationary policy $\eta_{k,\infty}(a|s)$, $B>0$ is some constant, $p_{k,\max}=\max\limits_{s,a} p_k(s,a)$,  $p_{k,\min}=\min \limits_{s,a} p_k(s,a)$. 
 \end{lemma}
 \begin{proof}
 See \cite{Q_rate}.
 \end{proof}
\begin{proof}[\textbf{Proof of Lemma \ref{lemma22:app_error}}]
 Using Chebyshev's inequality, we have  that for $\epsilon>0$,
\begin{align}
\label{eq:Prob_error}
    \mathbb{P}\left[\left(\left.|\sigma_{k,n}^2- \mathbb{E}(\sigma_{k,n}^2|\widetilde{\mathcal{H}}_{k,n})|\geq \epsilon \right)\right |\mathcal{H}_{k,n}\right]\leq \frac{\mathbb{V}(\left. \sigma_{k,n}^2\right|\mathcal{H}_{k,n})}{\epsilon^2}. 
\end{align}
When $\widetilde{\mathcal{H}}_{k,n}$ is given, the estimates $\left\{\widehat{Q}_{\ell,n}\right\}_{\ell\in\mathcal{N}_k}$ are  independent random variables. Let  $\mathcal{X}_{\ell}\triangleq \left(\widehat{Q}_{\ell,n}-\overline{Q}_{k,n}\right)^2$. Then, computing the \textit{variance of the sum}, we get 
\begin{align}
\label{eq:inter3}
   \mathbb{V}(\left. \sigma_{k,n}^2\right|\widetilde{\mathcal{H}}_{k,n})&= \frac{\sum\limits_{\ell\in\mathcal{N}_k}\mathbb{V}\left(\left.\mathcal{X}_{\ell}\right|\widetilde{\mathcal{H}}_{k,n}\right)}{(N_k-1)^2} \nonumber
   \\
   &+\sum\limits_{\substack{\ell,j\in\mathcal{N}_k\\ \ell\neq j}}\frac{\operatorname{Cov}(\left.\mathcal{X}_{\ell},\mathcal{X}_j\right|\widetilde{\mathcal{H}}_{k,n})}{(N_k-1)^2} \nonumber\\
   & \stackrel{\textnormal{(a)}}{\leq} \sum\limits_{\ell\in \mathcal{N}_k}\frac{\mathbb{E}\left[\left.\left(\mathcal{X}_{\ell}\right)^2\right|\widetilde{\mathcal{H}}_{k,n}\right]}{(N_k-1)^2}\nonumber\\
&+\sum\limits_{\substack{\ell,j\in\mathcal{N}_k\\ j\neq i}}\frac{\mathbb{E}\left[\left.\mathcal{X}_{\ell}\mathcal{X}_j\right|\widetilde{\mathcal{H}}_{k,n}\right]}{(N_k-1)^2},
\end{align}
 where in (a) we used the statistical properties that for any random variable $X$, $\mathbb{V}X \leq \mathbb{E} X^2$ and, for any $Y\geq 0, Z\geq 0$, $\operatorname{Cov}(Y,Z)\leq \mathbb{E}[Y  Z]$.  \par 
 Next, using Lemma \ref{lemma:Q_rate}, we have
 \begin{align}
\label{eq:q_Eq}
    &\sqrt{\mathcal{X}_{\ell}}\triangleq |\widehat{Q}_{\ell,n}(s,a)-\overline{Q}_{k,n}(s,a)|\nonumber\\
    &\leq |\widehat{Q}_{\ell,n}(s,a)-Q^{*}(s,a)| +|Q^{*}(s,a)-\overline{Q}_{k,n}|\nonumber\\
    &\leq\left(\frac{1}{n}\right)^{\frac{p_{k,\min}}{p_{k,\max}}\gamma}B +\frac{1}{N_k}\sum\limits_{i\in\mathcal{N}_k}\left(\frac{1}{n}\right)^{\frac{p_{k,\min}}{p_{k,\max}}\gamma}B \nonumber\\
     &\stackrel{\text{$(a)$}}{\leq}\left(\frac{1}{n}\right)^{\frac{p_{\min}}{p_{\max}}\gamma}B,   \nonumber\\
\end{align}
where in $(a)$ we use the relation $$\left(\frac{1}{n}\right)^{\frac{p_{k,\min}}{p_{k,\max}}}\leq\left(\frac{1}{n}\right)^{\frac{p_{\min}}{p_{\max}}}, \ \forall k \in \mathcal{N}_k.$$
  Substituting (\ref{eq:q_Eq}) into (\ref{eq:inter3}),  we get

\begin{align}
\label{eq:inter4}
   \mathbb{V}(\left. \sigma_{k,n}^2\right|\widetilde{\mathcal{H}}_{k,n}) \nonumber
   & \leq\frac{N_{k}}{(N_k-1)^2}\left(\frac{1}{n}\right)^{\frac{4p_{\min}}{p_{\max}}\gamma}(B)^4\nonumber\\
&+\frac{N_k(Nk-1) }{(N_k-1)^2}\left(\frac{1}{n}\right)^{\frac{4p_{\min}}{p_{\max}}\gamma}(B)^4, \nonumber\\
&=\frac{N_k^2 }{(N_k-1)^2}\left(\frac{1}{n}\right)^{\frac{4p_{\min}}{p_{\max}}\gamma}(B)^4
\end{align}
Finally, substituting \eqref{eq:inter4} into  (\ref{eq:Prob_error}) and setting $B' \triangleq B^2$
$$\epsilon=\frac{B'N_k}{(N_k-1)p_{\epsilon}}\left(\frac{1}{n}\right)^{\frac{2p_{\min}}{p_{\max}}\gamma}, $$ we get the result in (\ref{eq:finite_time}). 
\end{proof}
\end{appendices}




\bibliographystyle{IEEEtran}
\bibliography{refs}

\end{document}